\documentclass{article}
\usepackage{spconf,amsmath,graphicx,hyperref}
\usepackage{amsmath,amsfonts,bm,amssymb}
\usepackage{algorithmic}
\usepackage{algorithm}
\usepackage{array}
\usepackage[caption=false,font=normalsize,labelfont=sf,textfont=sf]{subfig}
\usepackage{textcomp}
\usepackage{stfloats}
\usepackage{url}
\usepackage{verbatim}
\usepackage{graphicx}
\usepackage{cite}
\usepackage{cite}
\usepackage{float}
\usepackage{amsmath}
\usepackage{amsthm}
\newtheorem{theorem}{Theorem}[section]

\newtheorem{lemma}[theorem]{Lemma}

\usepackage{xspace}
\newcommand{\ours}{eKSVD\xspace}
\newcommand{\ksvd}{KSVD\xspace}
\newcommand{\nystrom}{Nystr\"om\xspace}

\usepackage{xcolor}

\title{Kernel Singular Value Decomposition with 
\\ 
Extension  to  Multiple Data Sources}
\name{Xinjie Zeng, Qinghua Tao, Johan Suykens
\thanks{This work is jointly supported by the European Research Council under the European Union's Horizon 2020 research and innovation program/ERC Advanced Grant E-DUALITY (787960), iBOF project Tensor Tools for Taming the Curse (3E221427), KU Leuven C1 project Rethinking transformers through duality principles, Flanders AI Research Program, and Leuven.AI Institute.}}
\address{  KU Leuven - ESAT-STADIUS\\
 Kasteelpark Arenberg 10, B-3001 Heverlee - Belgium\\
 Email: firstname.lastname@esat.kuleuven.be, correspondence to Qinghua Tao\\
 }
\begin{document}
\maketitle
\begin{abstract}
Kernel Singular Value Decomposition (\ksvd) 
learns a pair of singular vectors w.r.t.\ an asymmetric kernel matrix, which can be induced by two data sources, e.g., the queries and keys in self-attention or the rows and columns of a given matrix. In this work, we extend \ksvd to multiple data sources, namely \ours,  which conducts joint  nonlinear feature learning  upon asymmetric kernels. In the primal formulation, the  projections associated with each data source  are jointly learned to capture  maximal information, while  incorporating pair-wise couplings. With the Lagrangian  and its Karush–Kuhn–Tucker (KKT) conditions,  the optimization in the dual leads to a generalization of the  shifted eigenvalue problem in Lanczos decomposition theorem of \ksvd. Further, a covariance-based  framework is  derived together with using neural networks (NNs) for explicit feature mappings, complementary to the kernel-based interpretation and optimization.
Numerical experiments verify the effectiveness of  our \ours  compared to methods based on Mercer kernels for tackling multiple data sources, and our innovation  of  deploying  NNs  demonstrates great flexibility for  kernel methods.
\end{abstract}
\section{Introduction}
\label{sec:intro}
Singular Value Decomposition (SVD) factorizes the given matrix through two sets of orthonormal eigenbases, i.e., left and right singular vectors\cite{stewart1993early}.
Its truncated form has been  proven an optimal low-rank matrix approximation in the Frobenius norm\cite{eckart1936}. However, SVD lacks the flexibility to capture nonlinear patterns, which are commonly encountered in practice. In \cite{suykens2016svd},
 a kernelized version of SVD (\ksvd) was proposed to extend SVD to feature spaces, introducing nonlinearity using asymmetric kernels. 

Kernel methods have long been the main workhorse in machine learning, and   commonly rely on  reproducing kernels upon Mercer's condition \cite{aronszajn1950theory},  associated with symmetric positive semi-definite
 kernels, including Support Vector Machines (SVMs) \cite{vapnik1999overview} and also the unsupervised tools of Kernel Principal Component Analysis (KPCA) \cite{scholkopf1998}, Kernel Canonical Correlation Analysis (KCCA) \cite{bach2002kernel},  and Kernel Partial Least Squares (KPLS) \cite{rosipal2001kernel}. In practice, asymmetry properties exist, e.g., directed graphs \cite{ou2016asymmetric}, biclustering \cite{kluger2003spectral}, and self-attention in Transformers \cite{chen2023primal,chen2024svgp}. Such asymmetry in data can hardly be well captured  by simply working with classical Mercer kernels. In contrast, the modeling and optimization in  \ksvd are formulated to accommodate asymmetric kernels 
in feature learning, making it promising to explore richer information  residing in the asymmetry of data.

SVD can be interpreted as jointly learning linear features w.r.t.\ the rows and columns of the given matrix through its left and right  singular vectors, respectively. More generally, \ksvd  extends SVD to nonlinear feature learning and can accommodate generic pairs of data sources, e.g., the queries and keys in self-attention \cite{chen2023primal,chen2024svgp}. In many applications, relevant information exists across multiple data sources, 
which may correspond to different representations of the same samples or to data from different learning tasks. 
In such settings, learning each data source independently is insufficient for satisfactory performance. Within classical kernel-based learning, various methods have been proposed to address joint  learning for varied data sources or tasks \cite{liu2023tensorized}, such as  MCCA~\cite{kettenring1971canonical} and 
MvKPLS~\cite{mvkpls}
which are still limited to Mercer kernels. Such applications motivate extending the joint feature learning perspective of KSVD beyond the two-source setting.

In this work, we formalize \ours, which
extends KSVD to multiple data sources and, meanwhile, retains the flexibility to learn underlying asymmetric relationships. 
In  modeling  with  the primal formulation, the projections w.r.t.\ each data
source are jointly learned to capture  maximal information in nonlinear feature spaces, while incorporating pair-wise couplings across all sources. Using the Lagrangian and the corresponding Karush–Kuhn–Tucker (KKT) conditions, we derive a dual formulation, leading to a generalization of the shifted eigenvalue problem associated with the Lanczos decomposition theorem of KSVD, where the couplings are reflected by different  asymmetric kernels. We further derive the covariance-based framework together with
 using neural networks (NNs) for explicit feature mappings, which complements kernel-based interpretations and
implementations.  Numerical experiments  verify the advantages of our \ours compared to methods built upon Mercer kernels and our parametric alternative with NNs  demonstrates promising potential for flexible modeling in kernel methods.

\section{Background}
\label{sec:background}

\noindent \textbf{SVD and Shifted Eigenvalue Problem.}
SVD  
factorizes a given matrix $ A \in \mathbb R^{N\times M}$ by two sets of  eigenbases, such that $A = U \Lambda  V^\top$.
The Lanczos decomposition theorem \cite{lanczos1958linear} establishes the equivalence between the solutions of SVD and the shifted eigenvalue problem, which is of particular importance to \ksvd and our work.
 
\begin{theorem}[Decomposition Theorem, Lanczos \cite{lanczos1958linear}]\label{theorem:lanczos}
Any non-zero rank $r$ matrix can be written as $ A=\tilde{ U}\tilde{ \Lambda}\tilde{ V}^\top$, where  $\tilde{ U}, \tilde{ \Lambda}, \tilde{ V}$ are defined by the shifted eigenvalue problem:
\begin{equation}\label{eq:theorem:lanczos}
\begin{array}{rl}
    \begin{aligned}
          A \tilde{ V}& =  \tilde{ U}\tilde{ \Lambda} \\
         { A}^\top\tilde{ U}& =  \tilde{ V}\tilde{ \Lambda}
    \end{aligned}
&  \Longleftrightarrow
\begin{bmatrix}
    \bm 0 &  A \\
    A^\top & \bm 0
\end{bmatrix}
\begin{bmatrix}
    \tilde{ U}\\
    \tilde{ V}
\end{bmatrix}
= 
\begin{bmatrix}
      \tilde{ U}\\
    \tilde{ V}
\end{bmatrix}\tilde{\Lambda},
\end{array}
\end{equation}
where  $\tilde{ U}\in \mathbb R^{N\times r}$ and  $\tilde{ V}\in \mathbb R^{M\times r}$ satisfy $\tilde{ U}^\top\tilde{ U}= I_r$ and $\tilde{ V}^\top\tilde{ V}= I_r$, and $\tilde{ \Lambda}\in \mathbb R^{r\times r}$ is diagonal with positive entries.
\end{theorem}

\noindent \textbf{Kernel SVD.} 
\ksvd conducts feature learning for two data sources $\mathcal{X} = \{\bm x_i \in \mathbb R^{d_x} \}$ and  $\mathcal{Z} = \{\bm z_j \in \mathbb R^{d_z} \}$ with an asymmetric kernel \cite{suykens2016svd,tao2024learning}. Hence, \ksvd can be regarded as not only an asymmetric generalization to  classical Mercer-kernel based methods, but also the nonlinear extension to the fundamental tool of SVD in linear algebra.
The optimization of \ksvd can be formalized with kernel-based framework under the setups of Least Squares SVMs (LSSVMs) \cite{suykens2016svd,tao2024learning}. In the primal, the optimization is cast as  jointly maximizing the variances w.r.t. two sets of projections with a  coupling term:
\begin{equation}\label{eq:ksvd:primal}
\begin{array}{rl}
    \max\limits_{\bm w, \bm v, e_i, r_j} &   \dfrac{\gamma}{2}(\sum\nolimits_{i=1}^N  e_i^2  + \sum\nolimits_{j=1}^M  r_j^2) -  \bm w^\top \bm v := J_{\rm KSVD}  \\
     {\rm s.t.} &   e_i = \bm w^\top \phi (\bm x_i), \  i=1, \ldots, N,\\
      &  r_j = \bm v^\top \psi (\bm z_j), \  j=1, \ldots, M.
\end{array}
\nonumber
\end{equation}
where $e_i,r_j$ denote the projection scores, $\bm w,\bm v$ denote the primal variables, $\phi,\psi$ denote the feature maps and $\gamma>0$ is the regularization constant. With the Lagrangian and KKT conditions, the optimization
in the dual form is attained as:
\begin{equation}\label{eq:ksvd:dual}
\begin{bmatrix}
    \bm 0 & K_{xz} \\
    K_{zx} & \bm 0
\end{bmatrix}
\begin{bmatrix}
    \bm \alpha\\
    \bm \beta
\end{bmatrix}
=\lambda 
\begin{bmatrix}
    \bm \alpha\\
    \bm \beta
\end{bmatrix},
\end{equation}
where $K_{xz} = K_{zx}^\top \in \mathbb R^{N\times M}$ denotes the kernel matrix allowing asymmetry with the $ij$-th entry $\left < \phi(\bm x_i) ,\psi(\bm z_j) \right >:= \kappa(\bm x_i, \bm z_j)$, $\bm \alpha,\bm \beta$ denote the Lagrangian multipliers, and $\lambda := 1/ \gamma$. 
By Theorem~\ref{theorem:lanczos}, the dual problem~\eqref{eq:ksvd:dual} cast as the shifted eigenvalue problem gives rise to the SVD of $K_{xz}$.

\section{The proposed methodology}
In this section, we extend \ksvd to multiple data sources and derive a generalization to the shifted eigenvalue problem upon asymmetric kernels, together with alternative optimization using covariances and NNs as explicit feature mappings. 

\subsection{Modeling with LSSVMs}
Given $Q\geq 2$ data sources  $\mathcal X_q = \{\bm x_{i_q }^{(q)} \in \mathbb R^{d_q} \}_{i_q=1}^{N_q}, q=1,...,Q$, we formulate the optimization of  \ours in the primal under LSSVM setups by jointly maximizing the projection variances of each  data source and meanwhile regularizing the couplings between each pair across  data sources:
\begin{equation}\label{eq:ours:primal}
\begin{array}{rl}
    \max\limits_{\bm w_q,  e_{i_q}^{(q)}}&  J := (\frac{Q}{2}-1)\sum\nolimits_{q=1}^Q \bm w_q^\top \bm w_q - \sum\nolimits_{\substack{1\leq p < q \leq Q}}\bm w_q ^\top \bm w_p   \\
    & + \frac{Q-1}{2}\gamma\sum\nolimits_{q=1}^Q\sum\nolimits_{i_q=1}^{N_q}  (e_{i_q}^{(q)} )^2  \\
    {\rm s.t.} &   e_{i_q}^{(q)} = \bm w_q^\top \phi_q (\bm x_{i_q}^{(q)}), \quad\forall i_q, \forall q,\\

\end{array}
\end{equation}
where $e_{i_q}^{(q)}$ is the projection score of the $q$-th data source with primal variables $\bm w_q \in \mathbb R^{d_\phi}$ and the feature map $\phi_q:\mathbb R^{d_q} \to \mathbb R^{d_\phi}$,  and $\gamma> 0$ is the regularization constant. 
We remark that formulation \eqref{eq:ours:primal} reduces to \ksvd 
with $Q=2$. Since \eqref{eq:ours:primal} only involves equality constraints, maximizing and minimizing $J$ yield the same first-order stationary points, up to the sign convention for the Lagrange multipliers. The maximization notation reflects the intended variance-maximization interpretation. As the objective is unbounded above, the subsequent analysis concerns stationarity rather than optimality.

\subsection{Optimization with Kernels}\label{sec:dual:generalized:shift}
In the following  Theorem \ref{theorem:generalized:lanzcos}, we prove that the optimization of \ours in the dual form is cast as an 
eigendecomposition on a structured matrix consisting of asymmetric kernel matrices, which reveal couplings across data sources. 
Thus, our derived Eqns. \eqref{eq:ours:primal} and \eqref{eq:ours:dual} together present a comprehensive primal-dual framework for model interpretation and optimization implementation.
\begin{theorem}\label{theorem:generalized:lanzcos}
The Lagrangian and its KKT conditions to \eqref{eq:ours:primal} lead to the following dual eigenvalue problem:
\begin{equation}\label{eq:ours:dual}
   \underbrace{
  \begin{bmatrix}
\bm 0 & K_{12} & \cdots& K_{1Q}  \\
K_{21} & \bm 0 & \ddots& \vdots \\
\vdots & \ddots & \ddots & K_{(Q-1)Q} \\
K_{Q1}& \cdots & K_{Q(Q-1)}& \bm 0\\
\end{bmatrix}}_{\triangleq \mathcal G_Q} 
 \begin{bmatrix}
 \bm  \alpha^{(1)}\\ 
 \bm \vdots \\
 \bm  \alpha^{(Q)} 
 \end{bmatrix}
 = 
 \lambda 
 \begin{bmatrix}
 \bm  \alpha^{(1)}\\ 
 \bm \vdots \\
 \bm  \alpha^{(Q)} 
 \end{bmatrix}
\end{equation}
with 
$\lambda = 1/ \gamma$ and kernel matrices
$K_{pq} \in \mathbb R^{N_p\times N_q}$ which are commonly asymmetric. We denote $K_{pq}:=K_{qp}^\top$ $\forall q,p =1, \ldots, Q, q\neq p$, and thus the matrix $\mathcal G_Q$ is symmetric. 

\end{theorem}

\begin{proof}
To formulate the Lagrangian, $Q$ vectors of dual variables, i.e., $\bm \alpha^{(q)} \in \mathbb R^{N_q}$, are introduced: 
\begin{equation}
    \label{eq:lagrange:obj:q}
    \begin{array}{rl}
\mathcal{L} =  J 
 - \sum\nolimits_{q=1}^Q\sum\nolimits_{i_q=1}^{N_q} \alpha_{i_q}^{(q)} ( e_{i_q}^{(q)} - \bm w_q^\top \phi_q(\bm x_{i_q}^{(q)})).
   \end{array}
   \nonumber
\end{equation}
For all $q$ and $ i_q$, the KKT conditions are
\begin{equation}\label{eq:kkt:q}
\left\{
\begin{aligned}
   \dfrac{\partial \mathcal{L}}{\partial \bm w_q}=0 \implies&
{\textstyle\sum\limits_{\substack{p=1\\p\neq q}}^Q}\bm w_p
-(Q-2)\bm w_q \!=
{\textstyle\sum\limits_{i_q=1}^{N_q}}
\alpha_{i_q}^{(q)}\phi_q(\bm x_{i_q}^{(q)})\\
    \dfrac{\partial \mathcal{L}}{\partial e_{i_q}^{(q)}} = 0 \implies & (Q-1){\gamma} e_{i_q}^{(q)} =  \alpha_{i_q}^{(q)},\\
   \dfrac{\partial \mathcal{L}}{\partial \alpha_{i_q}^{(q)}} = 0 \implies & \bm w_q^\top  \phi_q(\bm x_{i_q}^{(q)}) =  e_{i_q}^{(q)}.\\
\end{aligned}
\right.
\end{equation}

Eliminating 
$\bm w_q$ and $e_{i_{q}}^{(q)}$
in the  KKT conditions \eqref{eq:kkt:q} yields
\begin{equation}
 \sum\nolimits_{\substack{p=1,p\neq q}}^Q\sum\nolimits_{j_p=1}^{N_p} \alpha_{j_p}^{(p)} \phi_p(\bm x_{j_p}^{(p)})^\top \phi_q(\bm x_{i_q}^{(q)})   =  \alpha_{i_q}^{(q)} \lambda,
\end{equation}
whose matrix form is given as \eqref{eq:ours:dual}, completing the proof.
\end{proof}

We retain the eigenvectors associated with the largest positive
eigenvalues of $\mathcal G_Q$.
The corresponding primal and dual model representations are:
\begin{equation}
  e_{i_q}^{(q)} = \left \{
    \begin{aligned}
      & \bm w_q^\top \phi_q (\bm x_{i_q}^{(q)})\\
     &   [\frac{1}{Q-1}\sum\nolimits_{j=1, j\neq q}^Q K_{qj}\bm \alpha^{(j)}  ]_{i_q}
    \end{aligned}
      \right.
\end{equation}
The projection score in the dual involves  pair-wise couplings between data sources through the induced asymmetric kernel matrices $K_{qj}$, and it reduces to \ksvd form when only two data sources are addressed.

\subsection{Optimization with Covariances}
Complementary to kernel-based interpretation, Theorem \ref{thm:covariance:optimization} below formulates the optimization upon covariance matrices.  Although the resulting matrix $\mathcal C_Q$ in  \eqref{eq:covariance:optimization:equal} is in general nonsymmetric, its eigenvalues and eigenvectors are real, making it  a practical optimization alternative. 
\begin{theorem}\label{thm:covariance:optimization}
The Lagrangian and KKT conditions to \eqref{eq:ours:primal} lead to the covariance-based optimization as:
\predisplaypenalty=0
\begin{equation}\label{eq:covariance:optimization:equal}
\begin{array}{l}
   \underbrace{
  \begin{bmatrix}
\bm 0 & C_{22} &  C_{33}& \cdots & C_{QQ}  \\
C_{11} & \bm 0 & C_{33}& \cdots  & C_{QQ} \\
C_{11} & C_{22} & \bm 0 & \cdots &  \vdots\\
\vdots & \vdots & \vdots& \ddots  & C_{QQ} \\
C_{11}& C_{22} & C_{33}&  \cdots & \bm 0\\
\end{bmatrix}}_{\mathcal C_Q} 
 \begin{bmatrix}
 \bm  w_1\\ 
 \bm \vdots \\ 
 \bm  w_Q 
 \end{bmatrix}
 = 
\lambda
 \begin{bmatrix}
 \bm  w_1\\ 
 \bm \vdots \\
 \bm  w_Q
 \end{bmatrix}
 \end{array},
\end{equation}
where $C_{qq} \triangleq \sum\nolimits_{i_q=1}^{N_q}\phi_q(\bm x_{i_q}^{(q)})\phi_q(\bm x_{i_q}^{(q)})^\top \in \mathbb R^{d_\phi \times d_\phi}$ is the covariance matrix of the $q$-th data source, $q=1, \ldots, Q$. 

\end{theorem}

\begin{proof}
Eliminating the dual variables $\bm \alpha^{(q)}$ in  the first set of equations in  KKT conditions \eqref{eq:kkt:q}, for $q=1, \ldots, Q$ yields
  \begin{equation}
 \sum\limits_{\substack{p=1\\p\neq q}}^Q\bm w_p-(Q-2)\bm w_q = \gamma(Q-1)\sum\limits_{i_q=1}^{N_q} [\phi_q(\bm x_{i_q}^{(q)})\phi_q(\bm x_{i_q}^{(q)})^\top]\bm w_q, 
 \nonumber
\end{equation}
whose matrix form leads to \eqref{eq:covariance:optimization:equal}, completing the proof.
\end{proof}
\subsection{Optimization with Neural Networks}
While the formulations in \eqref{eq:ours:dual} and \eqref{eq:covariance:optimization:equal} provide solid mathematical grounding for \ours, these nonparametric implementations can be hindered by scalability as the sizes of $\mathcal G_Q$ and $\mathcal C_Q$ grow. 
We prove the zero-value property of the primal objective  in \eqref{eq:ours:primal}, when  evaluated at a 
stationary solution in \eqref{eq:ours:dual}. This motivates an alternative to efficiently work with  parametric neural networks as explicit feature mappings in an end-to-end manner, together with a task loss, i.e., $\min J^2 + \eta L_{\rm task}$, where $\eta$ is a tunable constant. When multiple components are retained, one could replace $J^2$ with $\sum_k J_k^2$, where $J_k$ denotes the primal objective associated with the $k$-th component.

\begin{lemma}\label{lemma:0}
    Solutions to the generalized shifted eigenvalue problem  \eqref{eq:ours:dual} lead to  $J=0$ in the primal optimization \eqref{eq:ours:primal}.
\end{lemma}

\begin{proof}
Multiplying the first KKT condition in~\eqref{eq:kkt:q} from the left by $\bm w_q^\top$, summing over $q$, we obtain
\begin{equation*}
2{\textstyle\sum\limits_{p<q}}\bm w_p^\top\bm w_q
-(Q-2){\textstyle\sum\limits_{q=1}^{Q}}\bm w_q^\top\bm w_q
\!=\!\gamma(Q-1)
{\textstyle\sum\limits_{q=1}^{Q}}
{\textstyle\sum\limits_{i_q=1}^{N_q}}
(e_{i_q}^{(q)})^2,
\end{equation*}
which can be rearranged as $2J=0$ by~\eqref{eq:ours:primal}.
\end{proof}

\section{Numerical Experiments}
\noindent \textbf{(1) Experimental setup.} 

\noindent In this section, we use seven 
benchmarks covering multimedia, text, and digit classification for numerical evaluations: Image-caption~\cite{kolenda2002independent} ($n$=1200), 3Sources~\cite{greene2009matrix} ($n$=169), YouTube video~\cite{misc_youtube_multiview_video_games_dataset_269} ($n$=2100), UCI Digit~\cite{brbic2018multi} ($n$=2000), BBC-3~\cite{greene2009matrix} ($n$=1268), Reuters-600~\cite{brbic2018multi} ($n$=600), and NUS-WIDE-OBJ (NUSOBJ)~\cite{Chua2009NUSWIDEAR, Long_2024_CVPR} ($n$=30000),  where $n$ is the number of data points. Here, Reuters-600 and NUSOBJ have five sources and the others contain three. 
Each dataset is split by  60/20/20 as training/validation/test sets.
We compare our \ours with KPCA, MCCA, and MvKPLS for feature learning and then take linear SVMs as downstream classifiers to report performance averaged over five seeds. The source features are concatenated for classification, while MCCA uses its learned canonical components. Hyperparameters such as kernel parameters and the number of components are selected by cross validation. For the larger NUSOBJ dataset, we use the Nystr\"om approximation \cite{tao2024learning} to make the kernel-based methods computationally tractable, and replace  the linear SVM with a single-layer linear NN  for speedups.

For \ours, we conduct three implementations: Dual (using kernels) based on \eqref{eq:ours:dual}, Covariance based on \eqref{eq:covariance:optimization:equal}, and Neural \ours (using parametric NNs) motivated by Lemma \ref{lemma:0}. For Dual and Covariance implementations, we first select  kernels  for each source and construct approximate feature maps by factorizing the corresponding kernel matrices. These 
approximated feature maps are then utilized to compute the asymmetric kernels and covariance matrices in the resulting decompositions, i.e., \eqref{eq:ours:dual} and  \eqref{eq:covariance:optimization:equal}. For Neural \ours, we use $L_{\operatorname{CE}}$, i.e., the cross-entropy loss, for classification tasks. 
Note that 
Neural \ours treats the regularization constant $\gamma$ as a learnable parameter, and we use a shared \(\gamma\) across all components for simplicity in a joint optimization of the NN model with the Adam optimizer.

\noindent \textbf{(2) Results and discussion.}

\noindent \ours and the compared methods generally conduct unsupervised learning and their learned features are evaluated by a downstream supervised classifier. Thanks to our flexibility of deploying parametric NNs,  the  supervision with labels is practical under the kernel-based framework of our \ours, leading to an end-to-end training of jointly learning explicit feature mappings and downstream classifier for  discriminative representations. 
Table~\ref{tab:exp:aligned} shows that under unsupervised setups  \ours outperforms the baseline methods on most datasets. Notably with parametrized feature maps using NNs, our Neural \ours consistently 
outperforms the baseline methods on six datasets, demonstrating our great flexibility together with solid model grounding upon kernel-based formulations. On NUSOBJ, Neural \ours falls slightly behind Dual \ours, which may partly reflect 
mutual effects during the joint  learning between 
the NN-based parametric feature mapping and \ours-related parameters.
We further investigate the training process of Neural \ours in Fig.~\ref{fig:neural-training-dynamics},
which is consistent with the zero-objective property in Lemma~\ref{lemma:0}.
\vspace{-0.2cm}

\begin{table}[ht!]
\centering
\caption{Test macro-F1 (\%), averaged over five splits. The highest mean is in bold; the second-highest is underlined; the ``$^{*}$'' denotes the case using \nystrom approximations.}
\label{tab:exp:aligned}
\fontsize{9}{11}\selectfont
\setlength{\tabcolsep}{2pt}
\renewcommand{\arraystretch}{1.12}
\begin{tabular*}{\columnwidth}{@{\extracolsep{\fill}}lrrrr@{\hspace{0.8em}}|r@{}}
\hline
& \multicolumn{4}{c@{\hspace{0.8em}}|}{Unsupervised} & \multicolumn{1}{c@{}}{Supervised} \\
Dataset & KPCA & MCCA & MvKPLS & \shortstack[r]{Dual\\\ours} & \shortstack[r]{Neural\\\ours} \\
\hline
Image-caption & 98.42 & 91.75 & \underline{98.83} & 98.09 & \textbf{99.08} \\
3Sources & 74.17 & 62.64 & 83.59 & \underline{83.88} & \textbf{91.08} \\
YouTube & 91.80 & 85.64 & 93.20 & \underline{94.00} & \textbf{94.87} \\
UCI Digit & 97.36 & 96.64 & 97.69 & \underline{97.80} & \textbf{98.81} \\
BBC-3 & 93.99 & \underline{94.83} & 93.80 & 94.52 & \textbf{97.04} \\
Reuters-600 & 70.95 & 75.92 & 65.18 & \underline{78.71} & \textbf{80.43} \\
NUSOBJ & \underline{38.16}\rlap{$^{*}$} & 30.46 & 33.06\rlap{$^{*}$} & \textbf{38.75}\rlap{$^{*}$} & 37.90 \\
\hline
\end{tabular*}
\end{table}

\begin{figure}[ht!]
\centering
\includegraphics[width=\columnwidth]{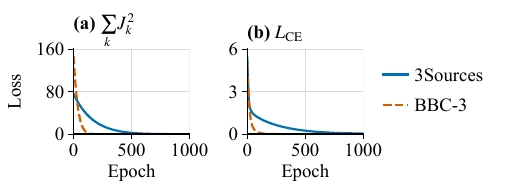}
\par\vspace{-8pt}
\caption{Training loss curves of Neural \ours on 3Sources (blue solid) and BBC-3 (orange dashed): (a) the \ours-related loss $\sum_k J_k^2$; (b) the classification task loss $L_{\operatorname{CE}}$.}
\label{fig:neural-training-dynamics}
\end{figure}

For applications requiring unsupervised  learning setups, we consider Dual and Covariance implementations. With identical feature maps, the implementation with covariances in \eqref{eq:covariance:optimization:equal} reproduces the performances with the dual form  in \eqref{eq:ours:dual}. We next examine their computational cost in Table~\ref{tab:exp:covariance-speed}
with $d_\phi$-dimensional feature maps constructed from random Fourier features for RBF kernels. 
Table~\ref{tab:exp:covariance-speed} shows that Covariance \ours
is faster in both training and inference given $N>d_\phi$. In training, given that each source contains $N$ samples, %
Covariance \ours solves a $Qd_\phi\times Qd_\phi$ eigenvalue problem in \eqref{eq:covariance:optimization:equal} whereas Dual \ours solves a $QN\times QN$ one in \eqref{eq:ours:dual}. At inference, Covariance \ours leverages explicit features, avoiding expensive kernel evaluations with large $N$. On NUSOBJ, Covariance \ours is much faster in training with $N=d_\phi$, as the Nystr\"om-based Dual \ours 
involves SVDs over larger feature matrices, containing all training samples. As for inference, the difference is negligible.

\begin{table}[t]
\centering
\caption{Average training and inference time ($\downarrow$). $N$ is the training sample size and $d_\phi$ the feature dimension. For NUSOBJ, N denotes the subset size for Nystr\"om approximation.}
\label{tab:exp:covariance-speed}
\fontsize{9}{11}\selectfont
\setlength{\tabcolsep}{2pt}
\renewcommand{\arraystretch}{1.12}
\begin{tabular*}{\columnwidth}{@{\extracolsep{\fill}}lrrrrrr@{}}
\hline
& & & \multicolumn{2}{c}{Dual} & \multicolumn{2}{c}{Covariance} \\
Dataset & $N$ & $d_\phi$ & \shortstack[r]{Training\\(s)} & \shortstack[r]{Inference\\(ms)} & \shortstack[r]{Training\\(s)} & \shortstack[r]{Inference\\(ms)} \\
\hline
Image-caption & 720 & 512 & 1.746 & 0.345 & \textbf{0.856} & \textbf{0.164} \\
3Sources & 101 & 512 & \textbf{0.020} & \textbf{0.108} & 0.662 & 0.367 \\
YouTube & 1260 & 512 & 8.667 & 0.731 & \textbf{1.044} & \textbf{0.144} \\
UCI Digit & 1200 & 512 & 6.180 & 0.514 & \textbf{0.917} & \textbf{0.205} \\
BBC-3 & 760 & 1024 & \textbf{1.851} & \textbf{0.740} & 6.076 & 0.817 \\
Reuters-600 & 360 & 1024 & \textbf{2.227} & \textbf{2.503} & 21.871 & 4.678 \\
NUSOBJ & 1024 & 1024 & 67.062 & 0.490 & \textbf{36.204} & \textbf{0.489} \\
\hline
\end{tabular*}
\end{table}

\section{Conclusion}
In this paper, we propose \ours, a multi-source extension of KSVD under the LSSVM framework. 
We derive a generalized shifted eigenvalue problem upon cross-source kernel matrices, together with a complementary covariance-based formulation. 
Thanks to the zero-objective property of the stationary solutions as proved,  
a parametric variant using NNs is established, which is capable of jointly learning explicit feature mappings and \ours projection weights. Experiments on real-world datasets demonstrate competitive performance and provide empirical guidance on solver selection for users. 
Nonetheless, 
the NN-based variant relies heavily on label supervision and hyperparameter tuning. Future work can proceed to explore the designs of self-supervised learning objectives to reduce reliance on labels.

\bibliographystyle{IEEEbib}
\bibliography{myrefs}

\end{document}